\documentclass[11pt]{article}

\usepackage[a4paper,margin=1in]{geometry}
\usepackage{amsmath,amssymb,amsthm}
\usepackage{booktabs}
\usepackage{graphicx}
\usepackage{lmodern}
\usepackage[T1]{fontenc}
\usepackage{microtype}
\usepackage[numbers]{natbib}
\usepackage{hyperref}
\hypersetup{colorlinks=true,linkcolor=black,citecolor=black,urlcolor=blue}

\newtheorem{proposition}{Proposition}
\newtheorem{definition}{Definition}
\theoremstyle{remark}
\newtheorem*{remark}{Remark}

\newcommand{\TEmarg}{\mathrm{TE}_{\mathrm{marg}}}
\newcommand{\TEcrn}{\mathrm{TE}_{\mathrm{crn}}}
\newcommand{\DE}{\mathrm{DE}}
\newcommand{\ME}{\mathrm{ME}}
\newcommand{\Ex}{\mathbb{E}}
\newcommand{\TV}{\mathrm{TV}}

\newif\ifanon
\ifdefined\anon \anontrue \fi

\title{Causal Attribution for Agentic Decisions:\\
Estimators, Coupling, and a Traceability Specification}
\author{\ifanon Anonymous author(s)\\Submission under double-blind review%
        \else Ajay Pravin Mahale\fi}
\date{\ifanon\relax\else Draft of \today\fi}

\begin{document}
\maketitle

\begin{abstract}
A provider of a high-risk AI system must keep records that make a decision
traceable, and for agentic systems it has not been established what those
records must contain for post-hoc causal attribution to be possible. We give the
estimator framework and then the conditions under which it fails. We separate the
marginal total effect that prior work measures from a common-random-numbers total
effect that isolates a step's own contribution, add the natural direct effect
under a pinned downstream, and check the estimators against hand derivations.
Both estimands then fail, in the same direction. Under the marginal estimand a
causally inert step has the identical total effect to the decisive one on every
run of our planted chain, an algebraic identity and not a coincidence at one
draw. Under common random numbers the decisive step returns exactly zero on the
runs where the executing step flips, about one in ten, while its direct effect
there is $0.25$ and it demonstrably acts; an exact zero does not certify that a
step did nothing, and we put that here rather than in the limitations. We derive
the coupling that keeps the direct effect estimable once contexts diverge, with a
closed form for its degradation, and show that the mediated share on which a
natural ranking is built is not a share under suppression: where the direct and
mediated paths oppose, it exceeds one and ranks a suppressed component above a
pure mediator. We publish the discrepancy experiment's pre-registration rather
than a result, because the live pipeline it requires was not available in the
study window. We contribute the traceability specification such a filing would
need, against a gap the Act's calendar opens: Article 86's right to an
explanation has applied since 2 August 2026, while the Article 12 logging and
Annex IV documentation that could evidence one were deferred to 2 December 2027
by Regulation (EU) 2026/1744.
\end{abstract}

\section{Introduction}
\label{sec:intro}

On 8 July 2026 the Union deferred the obligations this paper is about. Regulation
(EU) 2026/1744 \citep{omnibus2026} moved the application of Chapter III Sections 1
to 3 of the AI Act \citep{aiact2024},
which contain the technical documentation duty in Article 11, the record-keeping
duty in Article 12 and the transparency duty in Article 13, from 2 August 2026 to
2 December 2027 for Annex III high-risk systems. Recital (40) gives the
legislature's reason, and it is worth quoting because it is a statement about
evidence rather than about scheduling: the deferral responds to ``the delayed
availability of standards, common specifications, and alternative guidance and
the delayed establishment of national competent authorities'', which together
``lead to challenges that jeopardise the effective entry into application of those
obligations''.\footnote{Recital (40) is explanatory and non-binding. We cite it as
a recital and never as operative law.} The apparatus for demonstrating conformity
was not ready, on the Union's own record.

This paper is about one part of that apparatus. For an agentic system, a
sequence of model calls, retrievals and tool invocations that together produce a
regulated decision, what would a provider have to record for anyone to establish
afterwards which part of the system caused the outcome? The question is not
rhetorical. A survey of evidence tracing and execution provenance in LLM agents
grades the field's execution-provenance metrics as merely proposed, with no
agreed definitions and no adopted evaluation protocols \citep{survey2026}.

\paragraph{Where this sits relative to the closest work, in paragraph one rather
than in related work.} Causal attribution over agent trajectories already exists.
\citet{car2026} gives a five-operator intervention algebra and a Monte Carlo
Shapley estimator over agent steps, and we adopt its operators verbatim rather
than improving on them, since modifying an instrument one is testing hands a
reviewer a free rejection. Two things that work states about itself define the
space this paper occupies. It measures a total effect under which every
downstream step re-decides under fresh randomness, and it leaves the direct
effect as a refinement, naming common random numbers across divergent contexts as
the obstacle. And it operates on mocked tools, with real tools and their side
effects out of scope.

We take up both. Section~\ref{sec:two-te} shows the choice between the two total
effects is not a variance question. Under the marginal estimand a causally inert
step is numerically indistinguishable from the decisive one, and on our planted
chain that indistinguishability is an identity rather than a near miss. Under
common random numbers the inert steps return exactly zero, but the implication
runs one way only, and Proposition~\ref{prop:crn-degeneracy} gives the
conditions under which a decisive step returns exactly zero as well.
Section~\ref{sec:coupling} turns ``hard across divergent contexts'' into a closed
form, an optimal construction, a measured cost, and three implementation
requirements, one of which is that the replay floor is set by batch invariance in
the serving stack rather than by seed control \citep{he2025}.

\paragraph{Contributions.}
\begin{enumerate}
\item A separation result between two total effects on agent trajectories, with
  closed forms, a proof that the marginal estimand assigns an inert step and a
  decisive one the identical effect on every run of the planted chain, and a
  proof that the common-random-numbers estimand, which does separate them, still
  returns exactly zero for a decisive step on a measurable set of runs, so a zero
  is evidence of nothing on its own. Run against the incumbent locus rule of
  \citet{car2026} on the same chain, the two name different components on $52$
  of $60$ draws, in a direction we characterise rather than score
  (Section~\ref{sec:two-te}).
\item The natural direct effect under a pinned downstream, its decomposition
  identity, and pin plausibility reported per arm rather than pooled, the pooled
  figure being $1/2$ for every fidelity parameter and therefore uninformative
  (Section~\ref{sec:de}).
\item A closed form for shared-randomness coupling agreement, its gap against the
  optimal coupling, and the measured collapse of both under a sampler that sorts
  its vocabulary by probability (Section~\ref{sec:coupling}).
\item The observation that agent attribution has been built without reference to
  mediation methodology, worked through a pre-registered statistic of ours that
  ranked on a quantity that literature says not to compute
  (Section~\ref{sec:suppression}).
\item A pre-registered discrepancy experiment, published in full and not run,
  with its guards and one guard that was removed for being wrong
  (Section~\ref{sec:prereg}).
\item A traceability specification derived from what those estimators require,
  mapped onto the filed documents and onto two Commission instruments that do not
  yet exist (Section~\ref{sec:annexiv} and Appendix~\ref{app:spec}).
\end{enumerate}

\paragraph{One argument at two levels.} \ifanon Prior work asks\else A companion
study by the same author asks\fi{} the same question of circuit-level evidence and
finds that the claim a provider would file is not stable across defensible
analytic choices \citep{mahale2026multiplicity}. That result is about whether the evidence is
reproducible by a second analyst. This paper is about whether the evidence tracks
causation at all, one level of abstraction up, where the object is a trajectory
rather than a forward pass. The two failures are independent: evidence could be
perfectly stable and still track the wrong thing.

\paragraph{What we do not claim.} The empirical question, whether deployed
observability tracks causal effect, is specified and not answered. It needed a
live regulated pipeline that was not available in the study window. We say so
here, in the abstract, and in Section~\ref{sec:prereg}, and we report no rank
correlation in this paper as a finding.

\section{Setting and estimands}
\label{sec:setting}

A trajectory is a sequence of steps $\tau = (v_1,\dots,v_T)$ terminating in a
decision $Y$. Each step is typed: a model call, a tool call, a retrieval, a memory
read or write, or a branch. Step $k$ draws an action $a_k$ from a policy
conditioned on the prefix $a_{<k}$, using exogenous noise $u_k$. We write
$a^{\mathrm{fact}}$ and $u^{\mathrm{fact}}$ for the realised factual run.

Every effect below is a contrast against that run, as it must be: an effect is a
contrast against something that happened. This is what distinguishes the setting
from ordinary sensitivity analysis, and it is why the record a provider keeps
determines what can be estimated at all, which is the subject of
Section~\ref{sec:annexiv}.

\paragraph{Intervention operators.} We adopt the five operators of
\citet{car2026} verbatim: resampling a step's action under the unchanged policy,
forcing an action, forcing an observation, altering the context, and altering the
policy. Resampling is primary throughout. Removal is not used: taking a step out
puts the system off-distribution in the same unprincipled way that mean ablation
does for circuits in single-forward-pass interpretability, and the argument
against it there transfers directly. That instability is not incidental. At the
circuit level, the choice of ablation operator is one of seven analytic axes
across which a filed interpretability claim flips for 73.2\% of specification
pairs \citep{mahale2026multiplicity}; the operator choice we are making here is
the same class of choice, one level up.

\paragraph{The quantity that does not exist upstream.} In a single forward pass
there is no distinction between what a component contributed directly and what it
contributed by changing what a later component did, because there is no later
component that re-decides. In an agentic system there is. A step can have
near-zero direct effect and large total effect purely because it changed what a
subsequent step chose to do. That quantity is the object of
Sections~\ref{sec:de} and \ref{sec:suppression}, and no observability tool we are
aware of records what would be needed to compute it.

\section{Two total effects, and why the distinction is the crux}
\label{sec:two-te}

Fix a step $k$ and resample its action under the unchanged policy. Two estimands
follow, and they are not interchangeable.

\begin{definition}[Marginal total effect]
Intervene at $k$, then let every downstream step re-decide under \emph{fresh}
noise:
\[
  \TEmarg(k) \;=\; \Ex_{a'_k,\,u_{>k}}\!\left[\,Y \mid do(a_k := a'_k)\,\right]
              \;-\; y^{\mathrm{fact}} .
\]
\end{definition}

\begin{definition}[Common-random-numbers total effect]
Intervene at $k$, then let every downstream step re-decide under the
\emph{factual} noise $u_{>k} = u^{\mathrm{fact}}_{>k}$:
\[
  \TEcrn(k) \;=\; \Ex_{a'_k}\!\left[\,Y \mid do(a_k := a'_k),\,
                  U_{>k} = u^{\mathrm{fact}}_{>k}\,\right]
             \;-\; y^{\mathrm{fact}} .
\]
\end{definition}

The second is a counterfactual in Pearl's sense \citep{pearl2009}: same unit,
same exogenous noise, one variable changed. The first marginalises the noise
away and in doing so re-rolls every downstream decision. That is the confound
\citet{car2026} identifies and resolves with a point-of-commitment rule. We
resolve it by construction instead.

The construction is common random numbers, a standard variance-reduction
technique in simulation \citep{lawkelton2000}, and we claim no novelty in the
technique. Its use for rollout-based planning is recent \citep{yadav2026}. What
we claim is narrower and is the content of the next subsection: in this setting
the choice between marginalising the downstream noise and holding it fixed is
not a variance question at all. It decides whether a causally inert step is
distinguishable from a decisive one.

\subsection{A planted structure that separates them}

Consider a four-step chain with fidelity parameter $q$ and direct weight $w$:
\begin{align}
  a_0 &\sim \mathrm{Bern}(1/2) && \text{inert, no path to } y \nonumber\\
  a_1 &\sim \mathrm{Bern}(1/2) && \text{retrieval} \nonumber\\
  a_2 &\sim \mathrm{Bern}(1/2) && \text{inert, no path to } y \label{eq:scm}\\
  a_3 &= \begin{cases} a_1 & u_3 < q\\ 1-a_1 & \text{otherwise}\end{cases}
      && \text{executing tool call} \nonumber\\
  y   &= w\,a_1 + (1-w)\,a_3 . \nonumber
\end{align}
With $q = 0.9$, $w = 0.5$, and a factual run
$a^{\mathrm{fact}} = (0,1,1,1)$, $u_3^{\mathrm{fact}} = 0.770 < q$, so
$y^{\mathrm{fact}} = 1$. Figure~\ref{fig:two-te} and Table~\ref{tab:two-te} give the closed forms, all
derived by hand before any estimator was run.

\begin{figure}[t]
\centering
\includegraphics[width=\textwidth]{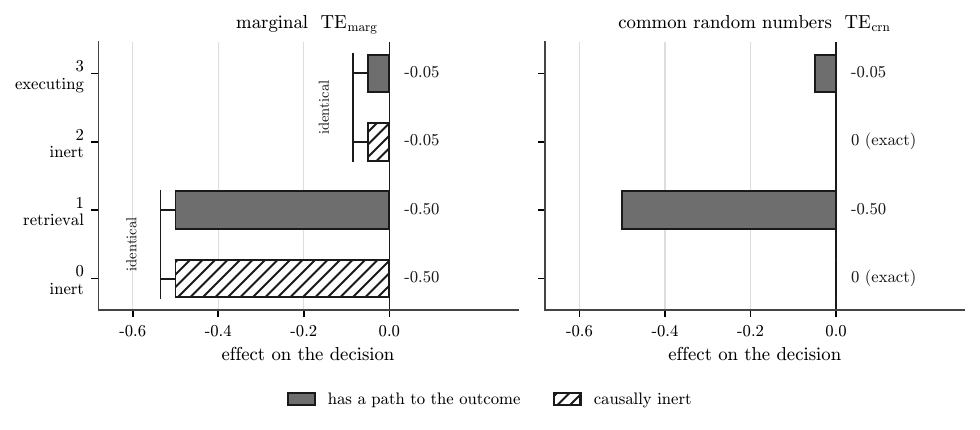}
\caption{The same four steps under the two estimands, at $q=0.9$, $w=0.5$,
\emph{conditional on the factual run stated in the text}
($a^{\mathrm{fact}} = (0,1,1,1)$, $u_3^{\mathrm{fact}} = 0.770 < q$). Hatched
bars are steps with no path to the outcome. Under the marginal estimand (left)
the inert step 0 is numerically identical to the decisive retrieval, and the
inert step 2 to the executing step; Proposition~\ref{prop:marg-identity} shows
both equalities hold on every run, not only this one. Under common random
numbers (right) both inert steps are exactly zero, not small. The converse does
not hold, and this figure cannot show that it does not:
Proposition~\ref{prop:crn-degeneracy} gives the runs, of probability $1-q$, on
which the decisive step 1 also returns exactly zero. Values are the closed forms
derived by hand before any estimator was run.}
\label{fig:two-te}
\end{figure}

\begin{table}[t]
\centering
\begin{tabular}{llrrrr}
\toprule
step & planted role & $\TEmarg$ & $\TEcrn$ & $\DE$ & $\ME$ \\
\midrule
0 & inert, no path to $y$ & $-0.50$ & $0$      & $0$      & $0$ \\
1 & decisive retrieval    & $-0.50$ & $-0.50$  & $-0.25$  & $-0.25$ \\
2 & inert, no path to $y$ & $-0.05$ & $0$      & $0$      & $0$ \\
3 & executing step        & $-0.05$ & $-0.05$  & $-0.05$  & $0$ \\
\bottomrule
\end{tabular}
\caption{Analytic values on the planted chain \eqref{eq:scm}, $q=0.9$, $w=0.5$,
conditional on the factual run $a^{\mathrm{fact}} = (0,1,1,1)$ with
$u_3^{\mathrm{fact}} = 0.770 < q$. Under the marginal estimand the inert step 0
is numerically indistinguishable from the decisive step 1, and the inert step 2
from the executing step 3. Under CRN both inert steps are exactly zero. The
$\TEcrn$ column is the one that depends on which factual run was drawn: had
$u_3^{\mathrm{fact}}$ landed above $q$, step 1 would read $0$ in this column
too (Proposition~\ref{prop:crn-degeneracy}).}
\label{tab:two-te}
\end{table}

The point of Figure~\ref{fig:two-te} is the left panel, and the reason is
stronger than one factual run. Both of its collisions are forced.

\begin{proposition}[Marginal indistinguishability is an identity]
\label{prop:marg-identity}
On the chain \eqref{eq:scm}, for every factual draw,
\[
  \TEmarg(0) \;=\; \TEmarg(1) \;=\; \tfrac{1}{2} - y^{\mathrm{fact}},
\]
\[
  \TEmarg(2) \;=\; \TEmarg(3) \;=\;
    w\,a^{\mathrm{fact}}_1 + (1-w)\bigl(q\,a^{\mathrm{fact}}_1
    + (1-q)(1-a^{\mathrm{fact}}_1)\bigr) - y^{\mathrm{fact}} .
\]
\end{proposition}

\begin{proof}
Forcing the inert $a_0$ leaves $a_1$ to redraw, so $\Ex[a_1] = \tfrac{1}{2}$ and
$\Ex[a_3] = q\Ex[a_1] + (1-q)(1-\Ex[a_1]) = \tfrac{1}{2}$, giving
$\Ex[Y] = \tfrac{1}{2}$. Forcing $a_1 := a'_1 \sim \mathrm{Bern}(1/2)$ with $a_0$
held gives $\Ex[Y] = w\Ex[a'_1] + (1-w)\Ex[a_3] = \tfrac{1}{2}$ by the same
computation, for any $w$. For the second pair, $a_1$ is held at
$a^{\mathrm{fact}}_1$; forcing the inert $a_2$ and letting $a_3$ redraw, and
forcing $a_3$ to a fresh policy draw at the same prefix, both leave
$\Ex[a_3] = q\,a^{\mathrm{fact}}_1 + (1-q)(1-a^{\mathrm{fact}}_1)$.
\end{proof}

So the left panel is not an artifact of the run we happened to draw. A causally
inert step and the decisive retrieval carry the \emph{same} marginal total effect
on every run of this chain, and no threshold on magnitude can separate them.
That is the case for a locus rule, and it is why \citet{car2026} needs one.

The right panel needs to be stated more carefully than we first stated it. Zero
under CRN is exactly zero and the inert steps do return it. The converse fails.

\begin{proposition}[An exact zero does not certify inertness]
\label{prop:crn-degeneracy}
On the chain \eqref{eq:scm} with $w = 1/2$, condition on the factual noise
$u_3^{\mathrm{fact}} \ge q$, an event of probability $1-q$. Then
$\TEcrn(1) = 0$ exactly, while
$\DE(1) = w\bigl(\tfrac{1}{2} - a^{\mathrm{fact}}_1\bigr) \ne 0$. For
$w \ne 1/2$ the conditional effect is nonzero whatever $u_3^{\mathrm{fact}}$ is.
\end{proposition}

\begin{proof}
Under CRN the executing step reuses $u_3^{\mathrm{fact}}$, so
$u_3^{\mathrm{fact}} \ge q$ forces $a_3 = 1 - a'_1$ for every forced $a'_1$, and
$y = w a'_1 + (1-w)(1-a'_1) = (1-w) + (2w-1)a'_1$, constant in $a'_1$ exactly
when $w = 1/2$. The factual run is one such branch, so $y^{\mathrm{fact}} = 1/2$
as well and the difference vanishes identically. The direct-effect arm instead
pins $a_3$ at $a^{\mathrm{fact}}_3$, leaving $y = w a'_1 + (1-w)a^{\mathrm{fact}}_3$,
which is not constant in $a'_1$ for any $w > 0$.
\end{proof}

Read plainly: on this chain the decisive retrieval returns a CRN total effect of
exactly zero on about one run in ten, and on those runs it is numerically
indistinguishable from the two steps with no path to the outcome at all, even
though its direct effect there is $0.25$ and it demonstrably acts. On those same
runs $y^{\mathrm{fact}} = 1/2$ sits at the interventional mean, so the marginal
effect is zero too and neither estimand in this paper separates the step. The
failure is not sampling noise that more rollouts would remove. It is exact
cancellation between a direct path of weight $w$ and a mediated path of weight
$1-w$ that the factual noise has inverted.

\texttt{scripts/validate\_crn\_degeneracy.py} measures this against the committed
estimator rather than against the derivation. Over $20\,000$ factual draws the
event $u_3^{\mathrm{fact}} \ge q$ occurs at rate $0.0974$ against a nominal
$0.10$; on the $300$-seed subset the estimator is run on it occurs $18$ times
(two-sided binomial $p = 0.02$ against $0.10$, a property of that seed set that
we report rather than reseed away), and on every one of those $18$ runs
$|\TEcrn(1)| < 10^{-12}$ while $|\DE(1)| = 0.254 \pm 0.014$. Sweeping $w$ over
$\{0.2, 0.3, 0.4, 0.5, 0.6, 0.8\}$ at $60$ draws each, the exact-zero rate is
$0.133$ at $w = 1/2$ and exactly $0$ at every other value.

We keep $w = 1/2$ as the reported configuration \emph{because} it is the worst
case for our own estimand, not despite it. Moving $w$ to a value at which the
cancellation cannot occur would be the repair our pre-registration forbids
elsewhere, and it would conceal the one condition under which a CRN zero must
not be read as inertness: a step whose direct and mediated paths carry equal
weight and oppose. The claim the rest of this paper rests on is therefore
one-directional. A causally inert step returns exactly zero under CRN on every
run. A step that returns exactly zero under CRN is not thereby inert, and a
provider reading a zero off a single trajectory has no warrant to conclude that
it is. Requirement R5 in Appendix~\ref{app:spec} exists for this reason: the
record has to carry the factual noise, so that the reader of an attribution can
tell which of the two cases produced the zero.

\subsection{Against the incumbent instrument}
\label{sec:vs-car}

Contribution~1 claims a separation from the locus rule of \citet{car2026}, and
until now nothing in this paper measured it. The rule is reproduced verbatim in
our code, as the latest step whose total-effect interval still excludes zero, and
\texttt{scripts/validate\_triage\_closeout.py} runs it against our estimand on the
chain \eqref{eq:scm} over $60$ factual draws. It returns the executing step~3 on
$60$ of $60$. The largest $|\TEcrn|$ returns the retrieval step~1 on $52$ of
$60$. The two disagree on $52$ of the $60$ draws.

We do not read that as the incumbent being wrong, and the paper would be weaker
if we did. The two rules answer different questions. A point-of-commitment rule
asks where the outcome became inevitable, and on this chain that is the tool call
that wrote the result. A common-random-numbers total effect asks which step's own
action carried the difference, and that is the retrieval whose output the tool
call executed. Both answers are defensible and they are not the same component. A
provider filing an attribution therefore has to say which question it answered,
because on a structure this simple the two disagree on seven runs in eight, and
nothing in any record the Regulation currently asks for would tell a reader which
was meant.

\paragraph{Resampling cannot attribute to a step that has stopped deciding.} The
intervention throughout is \texttt{do\_resample}: redraw the action at $k$ from
the unchanged policy. Where that policy has collapsed to a point mass, the redraw
returns the factual action every time and every effect at that step is exactly
zero, however decisive the step is. Setting $q = 1$ in \eqref{eq:scm} makes the
executing step a deterministic function of the retrieval, and on the committed
estimator its change rate is then exactly $0$ with $\TEmarg$, $\TEcrn$ and $\DE$
all exactly $0$, while the retrieval is untouched at $0.505$. This is
Proposition~\ref{prop:crn-degeneracy} seen from the other side and it has the
same remedy. The change rate is already computed for every step; a zero effect
has to be read next to it. A zero effect at a zero change rate is not a
measurement of no influence, it is the absence of a measurement.

\begin{remark}
This is a statement about estimators on a known structure. It is not a statement
about any deployed system, and we do not write it as one.
\end{remark}

\section{The direct effect and the decomposition}
\label{sec:de}

\begin{definition}[Natural direct effect, pinned downstream]
Intervene at $k$ and pin every $j > k$ to the action it took in the factual run:
\[
  \DE(k) \;=\; \Ex_{a'_k}\!\left[\,Y \mid do(a_k := a'_k),\,
               a_{>k} := a^{\mathrm{fact}}_{>k}\,\right]
          \;-\; y^{\mathrm{fact}} .
\]
\end{definition}

This is the natural direct effect of \citet{pearl2001}, with the textbook
treatment in \citet[\S4.5.4--4.5.5]{pearl2009}; the concept predates the
counterfactual formalisation and priority belongs to \citet{robins1992}.

The mediated effect is the residual,
\begin{equation}
  \ME(k) \;=\; \TEcrn(k) - \DE(k),
  \label{eq:identity}
\end{equation}
which is an identity rather than an approximation, and we say so because it is
worth being clear about what an identity can and cannot be evidence for. The
implementation computes $\ME$ as $\TEcrn - \DE$. Checking that
$\ME - (\TEcrn - \DE)$ is $5.55 \times 10^{-17}$, as an earlier version of
this paper reported, therefore measures the associativity of IEEE 754
subtraction and nothing about the estimators. We have removed that number. The
check that carries content is against the hand derivation, not against the code:
on the planted chain the mediated share of step 1 is
$|\ME|/|\TEcrn| = 0.25/0.5 = 0.5$, which must equal $1-w$ because the mediated
path carries exactly weight $1-w$, and it does; and each of $\TEmarg$,
$\TEcrn$ and $\DE$ separately matches its closed form at every step to within
Monte Carlo error, which is the comparison that could have failed.

\subsection{Pin plausibility, reported rather than assumed away}

The direct-effect arm at step 1 pins $a_3 = 1$ while $a_1$ has been resampled.
When $a'_1 = 1$ the pin agrees with the policy with probability $q$; when
$a'_1 = 0$ it agrees with probability $1-q$. Averaging over
$a'_1 \sim \mathrm{Bern}(1/2)$ gives $(q + 1 - q)/2 = 1/2$.

That average is $1/2$ for \emph{every} $q$, which is the whole objection to
reporting it. A quantity whose value does not move when the fidelity parameter
moves from $0.6$ to $0.99$ is not measuring the pin. An earlier version of this
paper reported the pooled figure as though it were informative; it is not, and
the conditional figures are. Measured on the committed estimator, the pooled
value is $0.483$ at $q = 0.9$, against $0.899$ on the arm where the intervention
left the action unchanged and $0.097$ on the arm where it changed it; at
$q = 0.99$ the same three numbers are $0.482$, $0.990$ and $0.010$.

The changed arm is the one the direct effect depends on, because it is the arm
in which the pin is doing work, and there the pinned continuation is one the
policy would have produced with probability $1-q$. At $q = 0.9$ that is one time
in ten. We therefore report pin plausibility per arm, and where the changed-arm
figure is small the direct-effect estimate rests on a continuation the policy
would rarely have produced, so $\ME$ must be read as bounded rather than as a
point estimate. Both arms are reported alongside every $\ME$.

\section{Coupling across divergent contexts}
\label{sec:coupling}

\citet[\S7]{car2026} states that a direct-effect arm needs common random numbers
across branches, calls this hard once contexts diverge, and leaves it as a
refinement. We give the closed form for what ``hard'' costs.

Token sampling is inverse-transform sampling from a categorical distribution. If
the uniform draw is a pure function of $(\text{run id}, \text{step},
\text{replicate})$ and never of the context, factual and counterfactual branches
consume identical draws. The noise is shared exactly. What is not shared exactly
is the benefit.

\subsection{Shared-\texorpdfstring{$u$}{u} coupling does not attain the bound}

A first implementation asserted that shared-$u$ inverse-transform sampling in
fixed index order attains the maximal-coupling bound $1 - \TV(p,q)$. That is
false, and the validation script rejected it at up to 215 Monte Carlo standard
errors before any of it reached this manuscript. Two branches agree at token $i$
exactly when $u$ lands in the intersection of the CDF intervals
$[P_{i-1},P_i)$ and $[Q_{i-1},Q_i)$, so
\begin{equation}
  A_{\mathrm{quant}}(p,q) \;=\;
    \sum_i \max\!\bigl(0,\; \min(P_i,Q_i) - \max(P_{i-1},Q_{i-1})\bigr)
  \;\le\; \sum_i \min(p_i,q_i) \;=\; 1 - \TV(p,q),
  \label{eq:quant}
\end{equation}
with equality only in degenerate cases. The mechanism behind the loss is that
mass moved between two tokens shifts every subsequent CDF boundary, so a
perturbation early in the index order decouples the entire tail.

Maximal coupling attains $1 - \TV(p,q)$, and that this is optimal is classical
\citep{levin2017}; we cite it rather than claim it. What we have not found
stated for this purpose, though it is an elementary computation, is the closed
form \eqref{eq:quant} for the shared-$u$ alternative and the size of the gap
between them at realistic divergence. Maximal coupling requires both $p$
and $q$ at draw time, so the factual branch's distribution must be carried
alongside the counterfactual one. That is one extra forward pass per step, and
it is available here because the factual run is being replayed anyway.

\begin{figure}[t]
\centering
\includegraphics[width=0.72\textwidth]{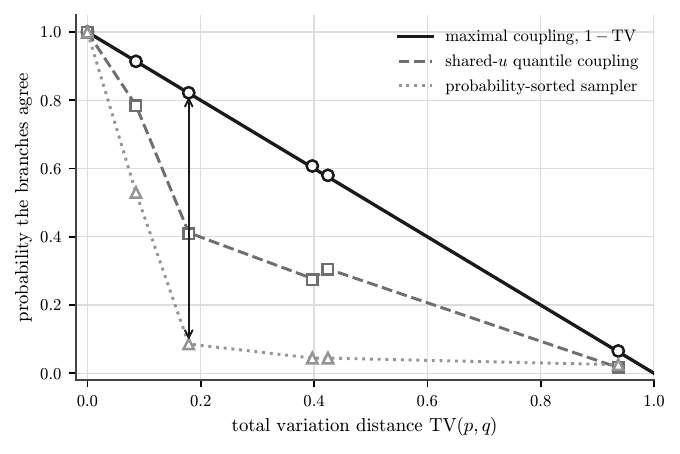}
\caption{Agreement probability against divergence, \emph{vocabulary 32},
40{,}000 draws per point. Lines are closed forms and markers are measured; every
measurement falls within 1.9 Monte Carlo standard errors of its closed form, and
the full table is in Appendix~\ref{app:coupling-table}. The arrow marks the cost
of the sampler's index order: at $\mathrm{TV}=0.18$ maximal coupling agrees
$82\%$ of the time and a sampler that sorts its vocabulary by probability agrees
$8.6\%$ of the time. Read the $82\%$ as general and everything else on this
figure as specific to $V=32$ and to this perturbation model: only maximal
coupling's $1-\mathrm{TV}$ is invariant, and the text below gives the range the
other two cover.}
\label{fig:coupling}
\end{figure}

Figure~\ref{fig:coupling} makes two design requirements binding and corrects a
third that an earlier version of this section stated too strongly.

\paragraph{Use maximal coupling, not quantile coupling.} At $\TV = 0.179$ and
$V = 32$ the gap is $0.821$ against $0.412$. This is the requirement that
survives everything below, and it survives because maximal coupling's agreement
is $1 - \TV$ by construction: it does not depend on the vocabulary, on the source
entropy, or on how the divergence was produced. Held at $\TV = 0.18$ across
vocabularies from $32$ to $10^5$ and three source entropies, it is $0.820$ in
every one of twelve cells, while quantile agreement runs from $0.631$ down to
$0.0044$, a factor of $142$, and probability-sorted agreement from $0.774$ down
to $0.0000$. At a production vocabulary the case for the extra forward pass is
therefore far stronger than this figure at $V = 32$ makes it look.

\paragraph{Keep the uniform draw a pure function of its key.} An advancing
generator makes the draw depend on call order, and call order differs between
branches.

\paragraph{The correction.} An earlier version said to fix the vocabulary order,
on the ground that sorting by probability collapses agreement, quoting $0.086$
against $0.412$. That inequality is not general and we withdraw the reasoning
behind the requirement while keeping a weaker version of the requirement itself.
When both branches are peaked they largely agree on which tokens are large, their
two sorted orders nearly coincide, and sorting \emph{raises} agreement: at
$V = 32$ with a peaked source it gives $0.774$ against $0.631$ for a fixed index
order, a reversal that survives a Monte Carlo run through the sampler itself.
What is wrong with sorting is not that it is worse. It is that the ordering
becomes a function of the branch's own distribution, so the coupling's cost is
set by the divergence rather than chosen by the implementer. The requirement that
holds is that the order be shared and branch-independent, or that maximal
coupling be used, which needs no order at all.

\paragraph{Two cautions on the numbers.} They are quoted at $V = 32$ and they do
not survive a change of vocabulary. They do not survive a change of perturbation
model either, at fixed vocabulary and fixed divergence: forming $q$ by adding
Gaussian logit noise of a stated scale gives quantile $0.412$ and sorted $0.086$
at $\TV = 0.179$, while bisecting a scale on a fixed noise vector to reach
$\TV = 0.180$ gives $0.595$ and $0.321$, a factor of $3.7$ on the second. The
result of this section is the ordering and the invariance of $1 - \TV$. The
magnitudes are illustrations of it and we do not offer them as estimates of
anything a deployed sampler would show.

\subsection{What replay actually requires}

The construction above assumes the underlying forward pass is reproducible. It
is not reproducible by default. \citet{he2025} establish that the forward pass
is already run-to-run deterministic and that the usual explanation, concurrency
racing with floating-point non-associativity, is not the cause; the cause is
lack of batch invariance under varying server load. Their measurement is that
1000 completions at temperature zero from one prompt gave 80 unique
completions, identical for the first 102 tokens and diverging at the 103rd,
collapsing to one completion under batch-invariant kernels.

Three consequences bind on any replay harness. Temperature zero with a fixed
seed is not sufficient, and a harness that assumes otherwise will silently
mis-measure its own replay floor. The remedy is batch-invariant kernels rather
than serialisation. And determinism does not survive a change of GPU or
framework version, so a replay is valid only within one pinned inference stack,
which is the same discipline we apply to the analysis environment.

Where batch invariance cannot be obtained, $\ME$ inherits a noise floor that is
a property of the serving infrastructure rather than of the system under study,
and no effect smaller than the measured null-replay floor may be reported as an
effect.

\begin{remark}
This is a result about samplers, established on synthetic categorical
distributions. It has not been run against a language model. The only quantity
in it that is model-independent is maximal coupling's $1 - \TV$. Every other
number here is as much a property of the vocabulary size, the source entropy and
the perturbation model as it is of the coupling.
\end{remark}

\section{Inconsistent mediation, and a pre-registered statistic built on it}
\label{sec:suppression}

A natural way to operationalise ``how much of this step's influence travelled
through what it caused later steps to do'' is the ratio $|\ME| / |\TEcrn|$. We
pre-registered a rank statistic on exactly that quantity. It is the wrong
quantity, and the mediation literature has said so for thirty years.

\subsection{The result is not ours}

In mediation analysis the ratio of the indirect effect to the total effect is
the \emph{proportion mediated}, and the configuration in which the direct and
indirect effects carry opposite signs is called \emph{inconsistent mediation},
with the mediator acting as a suppressor \citep{mackinnon2000,kenny_mediation}.
That the proportion mediated is unstable when the total effect is small, and can
exceed one or turn negative, is established in \citet{mackinnon1995}. We record
plainly that we did not derive this independently: we pre-registered a statistic
on the quantity, found the pathology on a constructed example, and located the
prior result on checking.

Restated in the notation of this paper, with no claim of novelty:

\begin{proposition}[Restatement of the proportion-mediated instability]
\label{prop:supp}
Let $\ME = \TEcrn - \DE$. If $\operatorname{sign}(\DE) \neq
\operatorname{sign}(\TEcrn)$, then
\[
  \frac{|\ME|}{|\TEcrn|} \;=\; \frac{|\TEcrn - \DE|}{|\TEcrn|}
  \;=\; 1 + \frac{|\DE|}{|\TEcrn|} \;>\; 1 ,
\]
unbounded above as $|\DE|$ grows in the opposing direction. A pure mediator has
$\DE = 0$ and attains exactly $1$, so a suppressed step outranks it.
\end{proposition}

\begin{proof}
$\ME$ is a difference, not a non-negative part, so opposite signs give
$|\TEcrn - \DE| = |\TEcrn| + |\DE|$. Divide by $|\TEcrn|$.
\end{proof}

\subsection{What is ours: the instrument, not the inequality}

Three things follow that the mediation literature had no occasion to state.

\paragraph{A pre-registered statistic ranked on a quantity that should not be
computed.} Our hypothesis H4 is a rank correlation on the mediated share.
Proposition~\ref{prop:supp} says that ranking places suppression above pure
mediation and mixes two opposite mechanisms into one score. The defect was in a
locked specification and was caught before any live data. We report it here
rather than in a footnote because the same construction is available to anyone
building attribution on a mediation decomposition, and the agent-attribution
literature has so far been built without reference to this methodology.

\paragraph{Inconsistent mediation has a concrete operational reading here.} In
an agent pipeline it is a retrieval that directly supports approval while
causing a later verification step to raise a flag. That is not a pathological
corner: it is a description of a component doing two opposing things, which is
common in a pipeline with a checking stage. Every structural model we used
before constructing the case had both paths pushing the same way, which is why
it had never been exercised.

\paragraph{The response is exclusion with a reported rate, not repair.} A share
is formed only where the decomposition splits a common-signed effect,
\[
  \operatorname{sign}(\DE) = \operatorname{sign}(\TEcrn)
  \quad\text{and}\quad
  |\DE| \le |\TEcrn| ,
\]
and is otherwise undefined and flagged. The rank statistic returns its exclusion
count alongside its value, so it cannot be reported without its exclusion rate.
Clamping $4/3$ to $1$ would repair an item that failed a validity check, which
our pre-registration forbids, and would merge suppression into pure mediation.
Steps with $|\TEcrn|$ below a pre-registered floor are undefined too, since a
share of no effect is undefined, but are not counted as suppression, so the
suppression rate is not inflated by inert steps.

\subsection{A worked instance}

Take the chain \eqref{eq:scm} with one sign flipped, $y = w\,a_1 - (1-w)\,a_3$,
at $w = 0.2$, $q = 0.9$. The same factual run gives
$y^{\mathrm{fact}} = -0.6$. Holding $u_3$ factual makes $a_3 = a'_1$, so
$y = (2w-1)a'_1$ and $\TEcrn(1) = +0.30$. Pinning $a_3 = 1$ gives
$\Ex[y] = -0.7$, so $\DE(1) = -0.10$, $\ME(1) = +0.40$ and the ratio is exactly
$4/3$. Measured over rollouts: $+0.3003$, $-0.1001$, $+0.4004$, ratio $1.3333$,
each within Monte Carlo error of the hand-derived value beside it. The
estimators are correct; the quantity built on them is not.

\begin{remark}[A previously computed value moved]
Applying the rule changed a rank correlation computed during pipeline
development from $\tau_b = +0.698$ to $+0.855$, with $80$ of $240$ steps
excluded, all inert and none suppressed. The earlier code assigned inert steps a
share of $0.0$, asserting their influence is entirely direct when they have no
influence at all. Those runs are synthetic smoke tests on planted generators
rather than results, and we report no rank correlation as a finding in this
paper. The rule follows from Proposition~\ref{prop:supp} and was not selected by
its effect on this number. We report the effect because a rule whose consequence
is withheld is worse than one whose consequence is awkward.
\end{remark}

\paragraph{H4 as pre-registered does not carry this guard, and we are not
repairing it silently.} The locked specification in
Section~\ref{sec:prereg} states H4 as a rank correlation on the per-step
mediated share with no exclusion rule attached, because it was written before
Proposition~\ref{prop:supp} was noticed. H4 has to be read with the rule above:
the share is undefined at suppressed and at inert steps, those steps are
dropped, and the drop rate is reported with the statistic. That is an
\emph{amendment to a locked pre-registration}. It is dated in
\texttt{RESEARCH\_LOG.md}, it was made before any live data exists rather than
after seeing a result, and it is recorded as an amendment rather than by editing
the pre-registration in place, because a pre-registration that can be edited is
not one. H4 was and remains exploratory. The primary contrast is H2 and it is
untouched.

\section{The pre-registered experiment, published and not run}
\label{sec:prereg}

The estimators above were built to answer one question: does the observability
record a provider would file recover the components that causally determined a
decision? That experiment is specified in full, committed with a timestamp, and
not run here. The live regulated pipeline it requires was not available in the
study window. We publish the specification rather than substitute a synthetic
proxy for the answer.

\subsection{Hypotheses}

Let $\tau_b$ denote Kendall's rank correlation with tie correction, computed per
decision and bootstrapped over decisions.

\begin{description}
\item[H1.] For each observability attributor, $\tau_b$ between its ranking and
  the causal ranking by $\TEcrn$ is low. Directional prediction: the bootstrap
  interval over decisions excludes $0.5$ from below.
\item[H2.] After conditioning on causal rank, attributor rankings retain
  systematic dependence on recency and verbosity. Directional prediction: the
  joint test of $(\beta_{\mathrm{rec}}, \beta_{\mathrm{verb}}) = (0,0)$ rejects.
\item[H3.] A non-trivial fraction of decisions contain a causally dominant step
  that the trace ranks negligible. Directional prediction: the fraction is
  positive with a Wilson interval excluding zero.
\item[H4.] The discrepancy concentrates in steps with high mediated and low
  direct effect. Directional prediction: the rank correlation between per-step
  mediated share and the observed-minus-causal rank gap is positive.
\end{description}

\paragraph{One primary contrast, and it is H2.} The pre-registered primary is the
joint two-degree-of-freedom Wald test of
$(\beta_{\mathrm{rec}}, \beta_{\mathrm{verb}}) = (0,0)$ in a Plackett-Luce model
conditioning on causal rank, pooled across attributors with the decision as the
clustering unit. Joint rather than two marginal tests, so it carries no
multiplicity correction. \textbf{H1, H3, H4 and every per-attributor breakdown
are exploratory} and are labelled as such wherever they appear.

\subsection{Specification}

\paragraph{Primary model.} Plackett-Luce rank-ordered, decision as the unit, with
a cluster-robust sandwich. Rank-ordered rather than ordinary least squares on
ranks because the latter is not a model of the ranking process, and because a
judge attributor yields only an ordering. The calibration ratio, mean standard
error over empirical standard deviation, is reported alongside.

\paragraph{Secondary, pre-specified.} CR1 cluster-robust least squares on the
attributor score, reported for comparability. If primary and secondary disagree,
that disagreement is stated in the abstract rather than in a robustness section.

\paragraph{Reporting convention.} Coefficients are reported as
$g = \beta / \lVert \beta \rVert_2$. Plackett-Luce is a random-utility model whose
noise scale is a normalisation rather than a measured fact, so a fit recovers
$c\beta$ for an unidentified $c > 0$ and levels are not reportable. Any ratio of
two coefficients cancels $c$ exactly, which is the right instinct, but anchoring
on the causal coefficient is not: H1 and H2 both predict that coefficient is
small, so the convention would become a ratio-of-normals problem
\citep{fieller1954} precisely when the finding is true. Dividing by the vector's
own norm is invariant to the same $c$, bounded in $[-1,1]$, and undefined only
if no attributor carries any signal, which is reported as itself. The interval
reported is the decision-level bootstrap percentile interval and not a
delta-method standard error, and the reason is a property of the map rather than
a preference. Differentiating $g_j = \beta_j / \lVert \beta \rVert_2$ gives
$\partial g_j / \partial \beta_k = (\delta_{jk} - g_j g_k) / \lVert \beta \rVert_2$,
so, using $\lVert g \rVert_2 = 1$,
\[
  \bigl\lVert \nabla_{\beta}\, g_j \bigr\rVert_2
  \;=\; \frac{\sqrt{1 - g_j^2}}{\lVert \beta \rVert_2} \,,
\]
which vanishes as $|g_j| \to 1$. The linear term degenerates exactly where the
constraint $|g_j| \le 1$ binds, and a symmetric interval built on it is not
confined to the parameter space it is an interval for. A percentile interval
inherits the constraint from the statistic and needs no repair.

\paragraph{Separation guard.} Quasi-complete separation is detected on the
unpenalised fit by two structural signals, neither of which scales with sample
size: $\max_j |\beta_j| > 25$, or a Hessian condition number above $10^{10}$. A
rule based on $|z| > 40$ was written and then removed, because $z$ grows like
$\sqrt{N}$ and a fixed cut fires on strong, well-identified effects: on a
reference design where the model recovers its own generative process and no
separation exists, $z$ reaches $19.5, 28.0, 40.8, 59.5$ and $86.0$ at $100, 200,
400, 800$ and $1600$ decisions. With a corpus in the hundreds it would have
flagged a genuine rejection as separation, and more readily the stronger the
result. \textbf{If separation is detected the primary contrast is reported as
indeterminate}, not refitted and reported as significant; a ridge-penalised fit
is shown as a bounded descriptive estimate with no $p$-value claimed, since on a
separated design the penalised fit still returns $p = 0$. The penalty is fixed in
advance and never tuned to a result.

\paragraph{Pre-fit disclosure.} Before fitting, the maximum absolute correlation
between each ranked attributor score and each covariate is reported. That
correlation is what produces separation, so reporting it lets a reader see the
risk before seeing the test.

\paragraph{Grids, not cells.} H3 is reported across the full grid of its two
thresholds rather than at one setting, because a single cell is a tuned number.

\subsection{Corpus and exclusion, as rules rather than numbers}

Two quantities are specified as decision procedures whose outputs are fixed
mechanically by a measurement that has not been taken, rather than as numbers we
would otherwise be choosing after the fact.

The corpus size is the smallest $n$ satisfying a minimum-detectable-effect
condition on the standardised scale, keyed to the within-decision residual
standard deviation measured in a null replay. The exclusion rule is that
decisions whose null-replay action-match rate falls below the measured replay
floor are excluded and the exclusion rate reported; and no effect smaller than
that floor is reported as an effect at all.

Both rules remain unresolved. That is the intended state: a rule removes the
researcher degree of freedom without pretending to a number nobody has measured,
and a later study resolves them by applying the rule rather than by judgment.

\subsection{Attributors, and one that was cut}

The observability attributors are span duration, output token count, and a
terminal-action indicator. Each is deterministic code traced to a documented
practice in a shipping tool. A recency attributor was specified in the original
design and \textbf{cut}: no shipping tool was found that ranks trace steps by
recency, so retaining it would have been the strawman the provenance audit
existed to prevent. Recency survives only as a covariate in the H2 regression,
where it is a property of presentation rather than a claimed practice.

\subsection{Why this is published rather than held}

A pre-registration that is never followed by data is usually a failure. We think
publishing this one is not, for three reasons. It states in advance what would
count as the discrepancy existing and what would count as it not existing, at a
level of detail that lets someone else run it. It records two guards that were
built, tested and in one case removed for being wrong, which is information a
later investigator would otherwise have to rediscover. And the estimators it
depends on are validated here against closed forms, so what remains unproven is
the empirical claim alone and not the machinery underneath it.

The claim we do not make is the one the experiment was designed to test. It is
stated as untested in the abstract, and nothing in this paper should be read as
evidence for or against it.

\section{What the Act requires, and the record it does not ask for}
\label{sec:annexiv}

Regulation (EU) 2024/1689 obliges a provider of a high-risk system to keep
technical documentation (Article 11 and Annex IV) and to design the system for
automatic logging over its lifetime (Article 12). We read both against the
question this paper asks, and the result is a clean separation. First, though,
the provision that makes the question urgent, and that we had left out.

\paragraph{The right that would need an attribution is in force; the record that
could support one is not.} Article~86 sits in Chapter IX, outside every deferral
in Article~113, and has applied since 2 August 2026. It gives ``any affected
person'' subject to a decision taken by the deployer on the basis of the output
of an Annex III high-risk system, point 2 excepted, the right to obtain from the
deployer ``clear and meaningful explanations of the role of the AI system in the
decision-making procedure and the main elements of the decision taken''.
Creditworthiness assessment is Annex III point 5(b) and is squarely inside that
scope. Three features of the text decide what could ever discharge it. The
duty-holder is the \emph{deployer}, not the provider who built the pipeline, so
the actor with the best access to the internals owes nothing under this Article.
The entitlement has two limbs joined by ``and'', and neither is the system's
processing: the \emph{role of the system in the procedure}, a procedural fact
about where the output entered and how much it decided, and the \emph{main
elements of the decision}, the substance of the human decision. There is no
third limb reaching how the output was produced. And the trigger is
per-instance and subjective, ``in a way that they consider to have an adverse
impact'', which is unusually claimant-friendly. The right therefore terminates at
exactly the layer at which causal attribution would have to begin.

Set that against the calendar. Article 11, Article 12 and Annex IV sit in
Chapter III Section 2 and, as amended by Regulation (EU) 2026/1744, apply to
Annex III systems only from 2 December 2027. Article~21(2), which lets a
competent authority demand access to the Article 12(1) logs, is Chapter III
Section 3 and was deferred to the same date. For roughly sixteen months the
Union has a live individual right to an explanation of an agentic credit
decision, an in-force duty on providers to investigate the causes of serious
incidents (Article~73(6), below), and no in-force obligation on anyone to have
kept a record from which either could be constructed. We do not argue that
Article~86 requires component-level attribution; on its own terms it does not.
We argue that the gap between what it asks for and what any record will contain
is not currently observable by anyone, because the record is not yet owed.

\paragraph{Description versus record, and general versus per-instance.} The most
component-aware sentence in the entire high-risk documentation package is Annex
IV point 2(c), which requires ``the description of the system architecture
explaining how software components build on or feed into each other and integrate
into the overall processing''. That is a design-time duty, discharged before
placing on the market: it yields a diagram of the pipeline in general. Article 12
yields events over the system's lifetime. Neither yields a record of how the
components related on one particular occasion. The distinction between a
\emph{description} and a \emph{record}, and between the \emph{general} and the
\emph{per-instance}, is what separates what is filed from what an attribution
would need.

\paragraph{The logging duty is bounded above by system-level purposes.} Article
12(2) fixes granularity by reference to three functions: identifying situations
that may present a risk within the meaning of Article 79(1), post-market
monitoring under Article 72, and deployer monitoring under Article 26(5). Having
read all three, none imports a component-level requirement. Article 79(1) asks
whether the system presents a risk. Article 72(2) asks for data on the system's
performance over its lifetime and, where relevant, ``an analysis of the
interaction with other AI systems'', which is interaction between systems rather
than between components inside one. A provider who logs enough to serve both
purposes can still be unable to say which retrieval, tool call or model step
determined a contested decision.

\paragraph{For credit scoring there is no minimum log content at all.} Article
12(3) is the only enumerated minimum in the Article and it binds only systems
under Annex III point 1(a), remote biometric identification. Creditworthiness is
Annex III point 5(b) and falls outside it. For that class, and every Annex III
class except biometric identification, the Regulation prescribes no minimum log
content whatsoever.

\paragraph{The delegation to harmonised standards points back at Article 12.}
Article~40(1) makes conformity with a harmonised standard whose reference has
been published in the \emph{Official Journal} a presumption of conformity with
the Section~2 requirements, so a reader may reasonably expect the operational
content of Article~12 to live in the standard rather than in the Regulation.
That is the strongest objection to the preceding paragraph and we put it to the
source. The operative instrument is Commission Implementing Decision
C(2025)~3871 of 23 June 2025, standardisation request M/613, which repealed and
replaced C(2023)~3215 of 22 May 2023: the earlier request had been drafted
against the \emph{proposal} for this Regulation rather than against the adopted
text, and its Annex~II still cross-referred to it as such. Item 3 of the current
Annex~I asks for a deliverable on ``record keeping through logging capabilities
by AI systems''. Its entire specification, Annex~II point~2.3, is this:

\begin{quote}
``The harmonised standards and standardisation deliverables in this area shall
set up specifications for record keeping. Those specifications shall
comprehensively cover all elements referred to in Article 12 of Regulation (EU)
2024/1689.''
\end{quote}

Thirty-two words, every one of them a reference. The request names no log content
of its own. It instructs CEN and Cenelec to cover the elements of Article 12,
which is the Article this section has just read as bounded above by three
system-level purposes, with an enumerated minimum that binds Annex~III
point~1(a) alone. The delegation is circular in content: the Regulation defers
the specification to the standard, and the request for the standard defers it
back to the Regulation. The same formula runs through the annex, at points 2.1,
2.2, 2.7, 2.8 and 2.9, so the brevity of 2.3 is not itself the finding. The
direction of the reference is.

Two caveats bound this, and the second is the one a standards reader will raise.
First, we found no reference to a harmonised standard under this Regulation
published in the \emph{Official Journal}, so the Article~40 presumption appears
to be available to nobody at present; we report that as what our search returned
rather than as a verified universal negative, and note that the Commission's own
register still records M/613 as under execution with an expiry date of 28
February 2027. Second, the drafting is well advanced. The CEN-CENELEC JTC~21
secretariat records that prEN~18229-1, ``AI Trustworthiness Framework, Part~1:
Logging'', was in public enquiry until 20 August 2026. Drafts at that stage are
not public documents, so what it requires is unknown to us and to any provider
reading the public record today. If it specifies log content at the level this
paper argues is missing, that narrows this paragraph and not the one before it,
and it does so by supplying content that no binding instrument asked for. What we
can say on the text that can currently be read is narrower than ``the Act is
silent'' and stronger than the paragraph above on its own: at both binding
layers, the question of what a log must contain is answered by a pointer to the
other layer.

\paragraph{The record expires twenty times sooner than the description.} Article
18(1) keeps the technical documentation at the disposal of competent authorities
for ten years. Article 19(1) and Article 26(6) keep the logs for ``at least six
months'', and both are limited to logs ``to the extent such logs are under their
control''. The static architectural description outlives the dynamic execution
record by a factor of twenty, and in a pipeline spanning several external
services no single actor need hold the whole trace. For a financial institution
the log duty is absorbed into existing financial-services record-keeping
(Article 19(2)), whose granularity was not designed for causal attribution.

\paragraph{The investigative ladder reaches the logs and stops at their content.}
Article~74 grants a market surveillance authority access to documentation and to
the training, validation and testing data sets, and access to source code only
on a reasoned request and only once procedures ``based on the data and
documentation provided by the provider have been exhausted or proved
insufficient''. Between those rungs sits Article~21(2), under which providers
shall, on a reasoned request by a competent authority, ``give the requesting
competent authority, as applicable, access to the automatically generated logs
of the high-risk AI system referred to in Article 12(1), to the extent such logs
are under their control''. An earlier draft of this section said execution traces
appear on neither rung. That was wrong, and correcting it sharpens the argument
rather than weakening it: the access right exists, and what does not exist, for
every Annex III class except remote biometric identification, is any prescribed
minimum of what those logs must contain. It is a right of access to a record
whose granularity the provider sets. Article~21 is itself in Chapter III
Section 3 and so was deferred to 2 December 2027 along with the logging duty it
reaches. Article 79(6)(b) then offers the finest granularity in the Act for
attributing a failure, ``a failure of a high-risk AI system to meet requirements
set out in Chapter III, Section 2'': one whole system against one whole Section.

\paragraph{There is an express duty to investigate causes, and it is in force
now.} Article~73(6) requires a provider, once a serious incident has been
reported, to ``perform the necessary investigations in relation to the serious
incident and the AI system concerned'', including a risk assessment and
corrective action, and forbids any investigation ``which involves altering the AI
system concerned in a way which may affect any subsequent evaluation of the
causes of the incident'' before the authorities have been informed. That is a
causal-investigation duty with an evidence-preservation clause attached. It sits
in Chapter IX and has applied since 2 August 2026. It is the closest the
Regulation comes to naming the question this paper asks, and it names no record
on which the investigation is to be conducted.

\paragraph{The Act has no vocabulary for this system class.} Across the six units
that constitute the high-risk documentation and record-keeping package, the words
\emph{agent}, \emph{agentic}, \emph{multi-agent}, \emph{orchestration} and
\emph{reasoning} do not occur. \emph{Traceability} occurs once, in the Article
12(2) chapeau, and is nowhere defined. The single occurrence of \emph{step} means
a step in the development process, not a step in an execution trace. Article 3(1)
defines an AI system in the singular throughout, so a pipeline of several
tool-backed services and a model is one AI system with one provider and one
deployer. Article~25 qualifies the actors but not the record: a distributor,
importer, deployer or third party is treated as a provider where it puts its
name on the system, makes a substantial modification, or modifies the intended
purpose so that the system becomes high-risk. That reallocates obligations along
the value chain. It does not create a per-component account of one execution,
and allocating responsibility between actors is a different thing from producing
evidence of which component caused an output.

We do not read any of this as prohibition. The Regulation mandates logging
capability sufficient for system-level risk and monitoring functions and leaves
granularity to the provider against the intended purpose. It simply nowhere sets
attribution of an output to an internal cause as the specification. Appendix~\ref{app:spec}
states what that specification would have to contain.

\paragraph{Two instruments that do not yet exist.} Article 11(1) as amended
obliges the Commission to produce a simplified Annex IV form for SMEs and small
mid-caps. Article 72(3) as amended obliges it to adopt guidance including a
template for the post-market monitoring plan, which is itself part of the Annex
IV documentation, by 2 September 2027. The obligations in Chapter III Section 2
apply from 2 December 2027. A specification published now arrives while both
instruments are still being drafted, which is the window in which it can be
adopted rather than merely cited.

\section{Limitations}
\label{sec:limitations}

The principal limitation is in the abstract: the discrepancy experiment is
specified and not run, because the live regulated pipeline it requires was not
available in the study window. Everything below is secondary to that.

\paragraph{Our own estimand returns an uninformative zero on a measurable set of
runs.} Proposition~\ref{prop:crn-degeneracy} is a limitation of the
common-random-numbers total effect and not only of the chain it is proved on.
Wherever a step's direct and mediated paths carry equal weight and the factual
noise inverts the mediator, the CRN effect is exactly zero for a step that acts.
On the planted chain that is one run in ten. We do not know the rate on a real
pipeline and have no way to estimate it from a single trajectory, which is the
setting the Regulation actually creates. The mitigation is in the specification
rather than in the estimator: R5 requires the factual noise to be recorded, so
that the two cases behind a zero can be told apart after the fact. Whether a
provider would in practice record it is exactly the question the unrun
experiment was meant to answer.

\paragraph{Validation rests on one family of planted generators.} The estimators
are checked against hand derivations on structural models we constructed. That
establishes correctness on known structure and nothing about behaviour on real
trajectories, where the step count is larger, the outcome is noisier, and the
policy is not a Bernoulli draw.

\paragraph{The coupling results have not met a language model.} Section~\ref{sec:coupling}
is established on synthetic categorical distributions. The closed forms are exact
and the measurements agree with them, but a real sampler brings a vocabulary four
orders of magnitude larger, a context that diverges progressively rather than at
one step, and a serving stack whose batch invariance we could not test. We now
know part of what the first of those costs: at a matched divergence, quantile
agreement at $V = 10^5$ is $0.0044$ against $0.631$ at $V = 32$. The
order-dependent numbers in Section~\ref{sec:coupling} are therefore illustrations
of an ordering, not estimates of what a deployed sampler would show, and we do
not report them as if they were.

\paragraph{No interaction arm.} Single-step interventions miss interactions
between steps. The Shapley value over agent steps is the principled treatment and
it is prior art \citep{car2026,castro2009}, so we cut it rather than reproduce it.
Where two steps interact, our per-step effects will not sum to the joint effect
and we do not report them as if they did.

\paragraph{One body of law bearing on Section~\ref{sec:annexiv} was not
retrieved.} Article~86(3) defers to rights ``otherwise provided for under Union
law'', which points at Regulation (EU) 2016/679, and the ten-year retention in
Appendix~\ref{app:spec} raises a storage-limitation question under it. We have
not analysed either interaction and do not claim to have.

An earlier version of this paragraph also listed the harmonised standards beneath
the Regulation as unread, and named the risk plainly: that our claim about
minimum log content might be true of the Regulation and false of the standard a
provider actually conforms to. That question is now addressed in
Section~\ref{sec:annexiv}, where the standardisation request turns out to
specify no log content of its own either. We record the sequence because the
objection was raised against this paper by an adversarial reader before it was
resolved, and because resolving it took two attempts: the first retrieval quoted
Implementing Decision C(2023)~3215, which had been repealed in June 2025, and the
error was caught by the Commission's own register rather than by a reading of the
draft. The standard itself remains unread, because drafts at enquiry are not
public. What we claim is bounded by that.

\paragraph{Two cleared references are absent from the bibliography.} Two papers we
consider relevant were verified to exist but our citation ledger records no author
list for either, and our own rule is that nothing enters the bibliography below
verified. They are omitted rather than cited from recollection.

\paragraph{The specification is unevaluated.} Appendix~\ref{app:spec} is derived
from what the estimators require, not from experience of filing it. No provider
has attempted conformance, and the cost estimates in it are reasoned rather than
measured. Requirement R6 in particular trades storage against interval width in a
way that only a deployment can settle.

\section{Related work}
\label{sec:related}

\paragraph{Post-hoc attribution over agent trajectories.} \citet{car2026} is the
closest work and is discussed in Section~\ref{sec:intro}. \citet{causalflow2026}
attributes failures and proposes repairs, but substitutes an oracle rather than
resampling under the same policy, and its causal responsibility score is a binary
indicator; its purpose is repair and training supervision.
\citet{agentracer2025} confirms oracle substitution for failed trajectories while
using programmed fault injection for successful ones, so its counterfactual is
approximated by an analyser conditioned on a ground-truth solution rather than by
a stored known-good output. \citet{ma2025} applies the Shapley value over
\emph{agents} rather than steps, on static logs, with no re-execution.
\citet{liao2026} audits provenance sensitivity by static ablation of context
factors, again without trajectory re-execution. \citet{whoandwhen2025} supplies
the benchmark the area evaluates against, framing attribution as identifying
which agent failed and at which step; we take the framing and not its headline
accuracy figure, which its own text shows to be an average over four evaluation
cells reported only in the abstract. None of these has a direct-effect
arm, and none compares an observability-derived ranking against a
causally-measured one.

\paragraph{Ex ante verification, a different object.} \citet{civex2026},
\citet{ebte2026}, \citet{attriguard2026} and \citet{causalarmor2026} all decide
whether a proposed action should execute, before it does. That is gatekeeping. This
paper is post-hoc attribution over a trajectory that has already run, and the two
do not substitute for one another: a guardrail that prevented nothing still leaves
a contested decision to explain.

\paragraph{Provenance and observability.} \citet{survey2026} surveys evidence
tracing and execution provenance and grades the field's metrics as proposed, with
no agreed definitions and no adopted protocols, which is the gap this paper's
specification addresses from the regulatory side. \citet{delegated2026} gives
semantic requirements and a schema for observability under delegated execution;
its non-identifiability result concerns the authorization relation and is a
scoping result rather than an impossibility barrier, and it contains no EU AI Act
content, so the regulatory bridge is built here rather than borrowed.

\paragraph{Mediation methodology.} The decomposition this paper uses is the
natural direct effect of \citet{pearl2001}, with the textbook treatment in
\citet{pearl2009} and priority for the estimand belonging to \citet{robins1992}.
The instability of the proportion mediated, and the naming of the opposite-signed
case as inconsistent mediation, are established in \citet{mackinnon1995} and
\citet{mackinnon2000}. That this literature is absent from the agent-attribution
work above is, we think, the more useful observation: the field has been
rebuilding mediation analysis without consulting it.

\paragraph{Determinism and replay.} \citet{he2025} establishes that the forward
pass is run-to-run deterministic and that inference nondeterminism arises from a
lack of batch invariance under varying load, which changes what a replay harness
must control. \citet{lawkelton2000} is the standard reference for common random
numbers; \citet{yadav2026} applies them to rollout-based planning, a different
problem with no coupling-agreement or attribution component. The optimality of
maximal coupling is classical \citep{levin2017}.

\bibliographystyle{plainnat}
\bibliography{refs}

\appendix

\section{A traceability specification for agentic high-risk systems}
\label{app:spec}

This appendix states what a provider must record for post-hoc causal attribution
over an agent trajectory to be possible at all. It is not a wish list. Every
requirement is a precondition for one of the estimands in
Sections~\ref{sec:two-te} to \ref{sec:suppression} to be computable from a filed
record, and where a requirement is not needed for that purpose it is not here.
Requirements are stated so that conformance is checkable by an auditor who has
the record and not the system.

\subsection{Conformance levels}

\begin{description}
\item[L0, filed today.] Annex IV technical documentation plus Article 12 logging
  at the granularity a provider chooses. Sufficient for system-level risk and
  monitoring. Insufficient for any of L1 to L3.
\item[L1, trajectory reconstructable.] The sequence of steps that produced a
  given decision can be recovered and ordered. Requirements R1 to R4.
\item[L2, counterfactually replayable.] A named step can be re-executed under a
  changed action while the rest of the trajectory is held at its factual
  randomness. Requirements R5 to R9. This is the level at which $\TEcrn$ and
  $\DE$ become estimable, and therefore the level at which ``which component
  caused this'' becomes a question with an answer.
\item[L3, attributable and contestable.] The record supports attribution that
  survives an adversarial reading, including its own uncertainty. Requirements
  R10 to R12.
\end{description}

The levels are cumulative and the jump that matters is L1 to L2. L1 is what most
observability stacks already deliver. L2 is what causal attribution requires and
what no provision of the Regulation currently asks for.

\subsection{Requirements}

\paragraph{R1, trajectory identity and order.} Every step carries a stable
trajectory identifier, a step index, and the identifier of the step whose output
it consumed. A decision is bound to exactly one trajectory identifier.
\emph{Why:} without a recoverable order there is no prefix to hold fixed, and
every estimand in this paper conditions on a prefix.

\paragraph{R2, step typing.} Each step is typed as model call, tool call,
retrieval, memory read, memory write, or branch. \emph{Why:} the intervention
operator that is valid at a step depends on its type, and a tool call with
external effect is not resampleable in the way a model call is.

\paragraph{R3, content commitment.} Inputs and outputs of every step are recorded,
or committed by cryptographic hash with the content retrievable for the retention
period. \emph{Why:} an intervention substitutes an alternative output at one step;
without the factual output there is nothing to substitute against, and without a
commitment an after-the-fact reconstruction is unfalsifiable.

\paragraph{R4, outcome function.} The decision $Y$ is produced by a rule that is
recorded, versioned and hashed, and that rule contains no model call.
\emph{Why:} if the outcome is itself produced by a model, its variance enters
every effect estimate and cannot be separated from the effect being measured.

\paragraph{R5, exogenous randomness as a keyed function.} Every sampling decision
records the key from which its randomness derived, and that key is a pure
function of (trajectory identifier, step index, replicate index) and of nothing
else, in particular not of the context. \emph{Why:} this is the whole of
Section~\ref{sec:coupling}. If the draw depends on call order, factual and
counterfactual branches consume different randomness, $\TEcrn$ silently becomes
$\TEmarg$, and by Table~\ref{tab:two-te} a causally inert step becomes
indistinguishable from the decisive one.

\paragraph{R6, sampling distribution at draw time.} For model steps, the
distribution over the sampled unit, or sufficient statistics for it, is retained.
\emph{Why:} maximal coupling requires both branches' distributions at draw time.
Without R6 only shared-$u$ quantile coupling is available, whose agreement is
bounded by \eqref{eq:quant} and, at the divergences measured in
Figure~\ref{fig:coupling}, can fall to half the achievable value. This is the most
expensive requirement here and the one most likely to be negotiated down; the
honest statement of the cost is that dropping it does not make attribution
impossible, it widens every mediated-effect interval by a measurable and
reportable amount.

\paragraph{R7, serving-stack fingerprint.} Model identifier and version,
quantisation, framework and kernel versions, accelerator model, and whether the
serving path is batch-invariant. \emph{Why:} determinism is not hardware or
version invariant \citep{he2025}, and temperature zero with a fixed seed is not
sufficient. Without R7 a replay divergence cannot be separated into a property of
the system under study and a property of the infrastructure it ran on.

\paragraph{R8, declared purity class per tool.} Each tool declares whether it is
pure, idempotent, effectful-reversible or effectful-unsafe, and the declaration is
part of the filed documentation. \emph{Why:} counterfactual re-execution must
refuse to run an effectful-unsafe tool, and a refusal must be recorded rather
than silently substituted, or the effect estimate is conditioned on an unrecorded
exclusion.

\paragraph{R9, environment binding.} External state that a tool read is captured
with the step, together with the time of reading. \emph{Why:} a counterfactual
evaluated against a different external state is not a counterfactual about the
system; it is a comparison of two different worlds.

\paragraph{R10, retention parity.} The execution record is retained for the same
period as the technical documentation it accompanies. \emph{Why:} Article 18(1)
sets ten years for the description and Article 19(1) six months for the record. If
attribution requires the record, the mandated evidence expires roughly twenty
times sooner than the accountability processes that would seek it, including
market surveillance under Article 74.

\paragraph{R11, cross-actor completeness.} Either one actor is accountable for the
whole trajectory, or every participating actor records a common join key and the
filed documentation names who holds which segment. \emph{Why:} Articles 19(1) and
26(6) both limit retention to logs ``to the extent such logs are under their
control'', which writes accountability diffusion into the retention duty itself. A
trajectory nobody is obliged to hold whole is a trajectory that cannot be
reassembled.

\paragraph{R12, signed effects and mandatory exclusion rates.} Where a derived
quantity is a ratio of effects, the underlying effects are recorded with sign, and
the derived quantity is reported with the count and rate of cases in which it was
undefined. \emph{Why:} Section~\ref{sec:suppression}. Under inconsistent mediation
the proportion mediated exceeds one, and a filing that reports the ratio without
its exclusion rate cannot be distinguished by a reader from one where the
exclusion never arose.

\subsection{What each requirement costs}

R1, R2, R4, R7, R8 and R12 are metadata and schema discipline: they change what is
written, not how much. R3 and R9 scale with trajectory volume and are the usual
subject of retention policy. R5 is free at design time and impossible to add
afterwards, which makes it the single most important requirement to adopt early.
R6 is the expensive one. R10 and R11 are policy rather than engineering, and R11
in particular cannot be satisfied by a provider acting alone.

\subsection{Relation to the filed documents}

R1 to R4 and R8 are naturally at home in Annex IV point 2(c), which already asks
for a description of how software components feed into each other; what they add
is the per-instance record corresponding to that description. R5 to R7 and R9
belong with the post-market monitoring plan, which Article 72(3) makes part of the
Annex IV documentation and for which the Commission owes a template by 2 September
2027. R10 and R11 cannot be met by a provider's own documentation and would
require the retention periods in Articles 19 and 26 to be revisited for systems
whose accountability depends on the record rather than on the description.

\subsection{What this specification does not claim}

It does not claim that a conforming record makes attribution accurate; it makes
attribution \emph{possible}, and the accuracy is then an empirical question of the
kind Section~\ref{sec:prereg} pre-registers. It does not claim the Regulation
forbids any of this. And it has not been evaluated against a live deployment, for
the reason stated in the abstract.

\section{Measured coupling agreement}
\label{app:coupling-table}

The values behind Figure~\ref{fig:coupling}. Every empirical column is within
1.9 Monte Carlo standard errors of the closed form beside it.

\begin{table}[h]
\centering
\small
\setlength{\tabcolsep}{5pt}
\begin{tabular}{lrrrrrr}
\toprule
logit shift & $\TV(p,q)$ & quant.\ bound & quant.\ emp. & max.\ bound &
max.\ emp. & prob-sorted \\
\midrule
0.00 & 0.0000 & 1.0000 & 1.0000 & 1.0000 & 1.0000 & 1.0000 \\
0.25 & 0.0856 & 0.7851 & 0.7839 & 0.9144 & 0.9142 & 0.5298 \\
0.50 & 0.1788 & 0.4121 & 0.4087 & 0.8212 & 0.8221 & 0.0858 \\
1.00 & 0.3969 & 0.2777 & 0.2740 & 0.6031 & 0.6075 & 0.0447 \\
2.00 & 0.4245 & 0.3057 & 0.3054 & 0.5755 & 0.5801 & 0.1655 \\
4.00 & 0.9369 & 0.0167 & 0.0177 & 0.0631 & 0.0654 & 0.0252 \\
\bottomrule
\end{tabular}
\caption{Vocabulary 32, 40{,}000 draws per point, $q$ formed by perturbing the
logits of $p$ with Gaussian noise of the stated scale. Every empirical value is
within 1.9 Monte Carlo standard errors of its closed form. The last column is
what a sampler that sorts the vocabulary by probability delivers. It is
\emph{not monotone} in the divergence: it falls to $0.0447$ at $\TV = 0.3969$
and rises again to $0.1655$ at $\TV = 0.4245$, because sorted agreement depends
on how far the two branches' orderings have moved relative to each other rather
than on how far the distributions have moved. That is the mechanism behind the
reversal reported in Section~\ref{sec:coupling}, visible here in the committed
data.}
\label{tab:coupling}
\end{table}

\paragraph{One cell of this table was wrong until 5 September 2026.} The
probability-sorted entry at logit shift $2.00$ read $0.0446$, a transcription of
the $0.0447$ on the row above it, where
\texttt{results/validation\_coupling.json} says $0.1655$. It had been there
since the table was written and it survived a full citation sweep, a figure pass,
two adversarial reviews and four readings, because a transcribed table is not
something a reader re-checks digit by digit. It was found by
\texttt{scripts/audit\_paper\_numbers.py}, which derives all forty-two cells of
this table from that artifact and fails the build when the manuscript disagrees
with it. We record the correction rather than making it silently: a paper about
evidentiary discipline that quietly repairs its own table has failed its own
test. The error did not touch any claim, because the sentence the column supports
is about ordering rather than magnitude, but no process in this repository would
have caught it before the audit existed.

\end{document}